\documentclass{article}
\usepackage{ijcai20}

\usepackage{times}
\usepackage{soul}
\usepackage{url}
\usepackage[hidelinks]{hyperref}
\usepackage[utf8]{inputenc}
\usepackage[small]{caption}
\usepackage{graphicx}
\usepackage{amsmath}
\usepackage{amsthm}
\usepackage{booktabs}
\usepackage{algorithm}
\usepackage{algorithmic}
\usepackage{multirow}
\usepackage{bigstrut}
\usepackage[]{placeins}
\usepackage{tabularx}
\usepackage{subcaption}
\usepackage{graphicx}
\usepackage{xcolor}
\usepackage{amsthm}
\newtheorem{defn}{Definition}
\newtheorem{asm}{Assumption}
\newtheorem{lem}{Lemma}

\title{Explainable Recommendation via Interpretable Feature Mapping and Evaluation of Explainability\footnote{Proceedings of the Twenty-Ninth International Joint Conference on Artificial Intelligence}}
\author{
Deng Pan\and
Xiangrui Li\and
Xin Li\And
Dongxiao Zhu\thanks{Corresponding author}\\
\affiliations
Department of Computer Science\\
Wayne State University, USA\\
\emails
\{pan.deng, xiangruili, xinlee, dzhu\}@wayne.edu
}

\begin{document}
 
\maketitle

\begin{abstract}
Latent factor collaborative filtering (CF) has been a widely used technique for recommender system by learning the semantic representations of users and items. Recently, explainable recommendation has attracted much attention from research community. However, trade-off exists between explainability and performance of the recommendation where metadata is often needed to alleviate the dilemma. We present a novel feature mapping approach that maps the uninterpretable general features onto the interpretable aspect features, achieving both satisfactory accuracy and explainability in the recommendations by simultaneous minimization of rating prediction loss and interpretation loss. To evaluate the explainability, we propose two new evaluation metrics specifically designed for aspect-level explanation using surrogate ground truth. Experimental results demonstrate a strong performance in both recommendation and explaining explanation, eliminating the need for metadata. Code is available from \url{https://github.com/pd90506/AMCF}.
\end{abstract}


\section{Introduction}
Since the inception of the Netflix Prize competition, latent factor collaborative filtering (CF) has been continuously adopted by various recommendation tasks due to its strong performance over other methods \cite{koren2009matrix}, which essentially employs a latent factor model such as matrix factorization
and/or neural networks
to learn user or item feature representations for rendering recommendations. Despite much success, latent factor CF approaches often suffer from the lack of interpretability \cite{zhang2018explainable}. In a contemporary recommender system, explaining why a user likes an item can be as important as the accuracy of the rating prediction itself \cite{zhang2018explainable}.

Explainable recommendation can improve transparency, persuasiveness and trustworthiness of the system \cite{zhang201919}. To make intuitive explanation for recommendations, recent efforts have been focused on using metadata such as user defined tags and topics from user review texts or item descriptions\cite{lu2018like,chen2018neural} to illuminate users preferences. Other works such as \cite{hou2019explainable,he2015trirank,zhang2014explicit} use \emph{aspects} to explain recommendations.
Although these approaches can explain recommendation using external metadata, the interpretability of the models themselves and the \emph{interpretable features} enabling the explainable recommendations have still not been systematically studied and thus, are poorly understood.
It is also worth mentioning that the challenges in explainable recommendation not only lie in the modeling itself, but also in the lack of a gold standard for evaluation of explainability. 

Here we propose a novel feature mapping strategy that not only enjoys the advantages of strong performance in latent factor models but also is capable of providing explainability via interpretable features. The main idea is that by mapping the {\it general features} learned using a base latent factor model onto interpretable {\it aspect features}, one could explain the outputs using the aspect features without compromising the recommendation performance of the base latent factor model. We also propose two new metrics for evaluating the quality of explanations in terms of a user's general preference over all items and the aspect preference to a specific item. Simply put,
we formulate the problem as: 1) how to find the interpretable aspect basis; 2) how to perform interpretable feature mapping; and 3) how to evaluate explanations. 

We summarize our main contributions as follows: 1) We propose a novel feature mapping approach to map the general uninterpretable features to interpretable aspect features, enabling explainability of the traditional latent factor models without metadata; 2) Borrowing strength across aspects, our approach is capable of alleviating the trade-off between recommendation performance and explainability; and 3) We propose new schemes for evaluating the quality of explanations in terms of both general user preference and specific user preference.

\section{Related Work}
There are varieties of strategies for rendering explainable recommendations. We first review methods that give explanations in light of aspects, which are closely related to our work. We then discuss other recent explainable recommendation works using metadata and knowledge in lieu of aspects.

\subsection{Aspect Based Explainable Recommendation}
Aspects can be viewed as explicit features of an item that could provide useful information in recommender systems. 
An array of approaches have been developed to render explainable recommendations at the aspect level using metadata such as user reviews.
These approaches mostly fall into three categories: 1) Graph-based approaches: they incorporate aspects as additional nodes in the user-item bipartite graph. For example, TriRank \cite{he2015trirank} extract aspects from user reviews and form a user-item-aspect tripartite graph with smoothness constraints, achieving a review-aware top-N recommendation. ReEL \cite{baral2018reel} calculate user-aspect bipartite from location-aspect bipartite graphs, which infer user preferences. 2) Approaches with aspects as regularizations or priors: they use the extracted aspects as additional regularizations for the factorization models. For example, AMF \cite{hou2019explainable} construct an additional user-aspect matrix and an item-aspect matrix from review texts, as regularizations for the original matrix factorization models. JMARS \cite{diao2014jointly} generalize probabilistic matrix factorization by incorporating user-aspect and movie-aspect priors, enhancing recommendation quality by jointly modeling aspects, ratings and sentiments from review texts. 3) Approaches with aspects as explicit factors: other than regularizing the factorization models, aspects can also be used as factors themselves. \cite{zhang2014explicit} propose an explicit factor model (EMF) that factorizes a rating matrix in terms of both predefined explicit features (i.e. aspects) as well as implicit features, rendering aspect-based explanations. Similarly, \cite{chen2016learning} extend EMF by applying tensor factorization on a more complex user-item-feature tensor.

\subsection{Beyond Aspect Explanation}
There are also other approaches that don't utilize aspects to explain recommendations. For example, \cite{lee2018explainable} give explanations in light of the movie similarities defined using movie characters and their interactions; \cite{wang2019explainable} propose explainable recommendations by exploiting knowledge graphs where paths are used to infer the underlying rationale of user-item interactions. With the increasingly available textual data from users and merchants, more approaches have been developed for explainable recommendation using metadata. For example, \cite{chen2019co-attentive,costa2018automatic,lu2018like} attempt to generate textual explanations directly whereas \cite{wu2019context,chen2019dynamic} give explanations by highlighting the most important words/phrases in the original reviews.

Overall, most of the approaches discussed in this section rely on metadata and/or external knowledge to give explanations without interpreting the model itself. In contrast, our Attentive Multitask Collaborative Filtering (AMCF) approach maps uninterpretable general features to interpretable aspect features using an existing aspect definition, as such it not only gives explanations for users, but also learns interpretable features for the modelers. Moreover, it is possible to adopt any latent factor models as the base model to derive the general features for the proposed feature mapping approach. 


\section{The Proposed AMCF Model}\label{sec:model}
In this section, we first introduce the problem formulation and the underlying assumptions. We then present our AMCF approach for explainable recommendations. AMCF incorporates aspect information and maps the latent features of items to the aspect feature space using an attention mechanism. With this mapping, we can explain recommendations of AMCF from the aspect perspective. An \textbf{Aspect} $\boldsymbol{s}$ \cite{bauman2017aspect} is an attribute that characterizes an item. Assuming there are totally $m$ aspects in consideration, if an item has aspects $s_{i_1}, ..., s_{i_k}$ simultaneously, an item can then be described by $ \mathcal{I}_i = \{s_{i_1}, s_{i_2},..., s_{i_k}\}$, $k\leq m$.
We say that an item $i$ has aspect $s_j$, if $s_j\in\mathcal{I}_i$.


\subsection{Problem Formulation}


\textbf{Inputs}: The inputs consist of $3$ parts: the set of users $U$, the set of items $V$, and the set of corresponding multi-hot aspect vectors for items, denoted by $S$.

\noindent\textbf{Outputs}: Given the user-item-aspect triplet, e.g. user $i$, item $j$, and aspect multi-hot vector $s_j$ for item $j$, our model not only predicts the review rating, but also the user general preference over all items and the user specific preference on item $j$ in terms of aspects, i.e., which aspects of the item $j$ that the user $i$ is mostly interested in.

\subsection{Rationale}
The trade-off between model interpretability and performance states that we can either achieve high interpretability with simpler models or high performance with more complex models that are generally harder to interpret \cite{zhang2018explainable}. Recent works \cite{zhang2018explainable,he2015trirank,zhang2014explicit} have shown that with adequate metadata and knowledge, it is possible to achieve both explainability and high accuracy in the same model. 
However, those approaches mainly focus on explanation of the recommendation, rather than exploiting the interpretability of the models and features, and hence are still not interpretable from modeling perspective. Explainability and interpretability refer to ``why" and ``how" a recommendation is made, respectively. Many above-referenced works only answer the ``why" question via constraints from external knowledge without addressing ``how". Whereas our proposed AMCF model answers both ``why" and ``how" questions, i.e., our recommendations are made based on the attention weights (why) and the weights are learned by interpretable feature decomposition (how). To achieve this, we assume that an interpretable aspect feature representation can be mathematically derived from the corresponding general feature representation. More formally:


\begin{asm}\label{asm:formal}
Assume there are two representations for the same prediction task: $\boldsymbol{u}$ in complex feature space $\mathcal{U}$ (i.e. general embedding space including item embedding and aspect embedding), and $\boldsymbol{v}$ in simpler feature space $\mathcal{V}$ (i.e. space spanned by aspect embeddings), and $\mathcal{V}\subset \mathcal{U}$. We say that $\boldsymbol{v}$ is the projection of $\boldsymbol{u}$ from space $\mathcal{U}$ to space $\mathcal{V}$, and there exists a mapping $\boldsymbol{M(\cdot, \theta)}$, such that $\boldsymbol{v} = \boldsymbol{M(\boldsymbol{u}, \theta)}$, with $\theta$ as a hyper-parameter.
\end{asm}

This assumption is based on the widely accepted notion that a simple local approximation can give good interpretation of a complex model in that particular neighborhood \cite{ribeiro2016should}. Instead of selecting surrogate interpretable simple models (such as linear models), we map the general complex features to the simpler interpretable aspect features, 
then render recommendation based on those general complex features. We give explanations using interpretable aspect features,
achieving the best of both worlds in keeping the high performance of the complex model as well as gaining the interpretability of the simpler model. In this work, the \textit{interpretable simple features} are obtained based on \textit{aspects}, hence we call the corresponding feature space as \textit{aspect space}. To map the complex general features onto the interpretable aspect space, we define the aspect projection.

\begin{defn}\label{def:fp}
\textbf{(Aspect Projection)} Given Assumption~\ref{asm:formal}, we say $\boldsymbol{v}$ is an aspect projection of $\boldsymbol{u}$ from general feature space $\mathcal{U}$ to aspect feature space $\mathcal{V}$ (Figure~\ref{fig:int_space}).
\end{defn}

To achieve good interpretability and performance in the same model, from Definition~\ref{def:fp} and Assumption~\ref{asm:formal}, we need to find the mapping $\boldsymbol{M(\cdot, \theta)}$. Here we first use a latent factor model as the base model for explicit rating prediction, which learns general features, as shown in Figure~\ref{fig:amcf} (left, $L_{pred}$), where we call the \textit{item embedding} $\boldsymbol{u}$ as the general complex feature learned by the base model. Then the remaining problem is to derive the mapping from the non-interpretable general features to the interpretable aspect features. 

\subsection{Aspect Embedding}
\label{sec:ext}
To design a simple interpretable model, its features should be well aligned to our interest, e.g. the \emph{aspects} is a reasonable choice. 
Taking movie genre as an example: if we use $4$ genres (Romance, Comedy, Thriller, Fantasy) as $4$ aspects, the movie \emph{Titanic}'s aspect should be represented by $(1, 0, 0, 0)$ because it's romance genre, and the movie \emph{Cinderella}'s aspect is $(1, 0, 0, 1)$ because it's genre falls into both romance and fantasy. 

From Assumption~\ref{asm:formal} and Definition~\ref{def:fp}, to make the feature mapping from a general feature $\boldsymbol{u}$ to an aspect feature  $\boldsymbol{v}$, we need to first define the aspect space $\mathcal{V}$. Assuming there are $m$ aspects in consideration, we represent the $m$ aspects by $m$ latent vectors in general space $\mathcal{U}$, and use these $m$ aspect vectors as the basis that spans the aspect space $\mathcal{V}\subset \mathcal{U}$. These aspects' latent vectors can be learned by neural embedding or other feature learning methods, with each aspect corresponding to an individual latent feature vector. Our model uses embedding approach to extract $m$ aspect latent vectors of $n$-dimension, where $n$ is the dimension of space $\mathcal{U}$. In Figure~\ref{fig:amcf}, the vertical columns in red ($\boldsymbol{\psi}_1, ..., \boldsymbol{\psi}_m$) represent $m$ aspect embeddings in the general space $\mathcal{U}$, which is obtained by embedding the aspect multi-hot vectors from input.


\subsection{Aspect Projection of Item Embedding}
\label{sec:att}
In Assumption~\ref{asm:formal}, $\boldsymbol{u}$ is the general feature representation (i.e. the item embedding) in space $\mathcal{U}$, and $\boldsymbol{v}$ is the interpretable aspect feature representation in space $\mathcal{V}$. The orthogonal projection from the general space $\mathcal{U}$ to the aspect space $\mathcal{V}$ is denoted by $\boldsymbol{M}$, i.e. $\boldsymbol{v} = \boldsymbol{M}(\boldsymbol{u})$. 

From the perspective of learning disentangled representations, the item embedding $\boldsymbol{u}$ can be disentangled as $\boldsymbol{u}=\boldsymbol{v}+\boldsymbol{b}$ (Figure~\ref{fig:int_space}), where $\boldsymbol{v}$ encodes the aspect information of an item and $\boldsymbol{b}$ is the item-unique information. For example, movies from the same genre share similar artistic style ($\boldsymbol{v}$) yet each movie has its own unique characteristics ($\boldsymbol{b}$). With this disentanglement of item embeddings, we can explain recommendation via capturing user's preference in terms of aspects.

Let's assume that we have $m$ linearly independent and normalized \textit{aspect} vectors $(\boldsymbol{\psi}_1, ..., \boldsymbol{\psi}_m)$ in space $\mathcal{U}$, which span subspace $\mathcal{V}$. For any vector $\boldsymbol{v} = \boldsymbol{M}(\boldsymbol{u})$ in space $\mathcal{V}$, there exists an unique decomposition such that $\boldsymbol{v} = \sum_{i=1}^m v_i \boldsymbol{\psi}_i$. 
The coefficients can be directly calculated by $v_i = \boldsymbol{v}\cdot \boldsymbol{\psi}_i = \boldsymbol{u}\cdot \boldsymbol{\psi}_i$, ($i = 1, ..., m$, $\boldsymbol{\psi}_i$ is normalized). Note that the second equality comes from the fact that $\boldsymbol{v}$ is the orthogonal projection of $\boldsymbol{u}$ on space $\mathcal{V}$.

Generally speaking, however, $(\boldsymbol{\psi}_1, ..., \boldsymbol{\psi}_m)$ are not orthogonal. In this case, as long as they are linearly independent, we can perform Gram-Schmidt orthogonalization process to obtain the corresponding orthogonal basis. The procedure can be simply described as follows:
    $\tilde{\boldsymbol{\psi}}_1 = \boldsymbol{\psi}_1; \text{ and }
    \tilde{\boldsymbol{\psi}}_i = \boldsymbol{\psi}_i - \sum_{j=1}^{i-1} \langle\boldsymbol{\psi}_i, \tilde{\boldsymbol{\psi}}_j\rangle \tilde{\boldsymbol{\psi}}_j,$
where $\langle\boldsymbol{\psi}_i, \tilde{\boldsymbol{\psi}}_j\rangle$ denotes inner product. We can then calculate the unique decomposition as in the orthogonal cases. Assume the resulting decomposition is $\boldsymbol{v} = \sum_{i=1}^m \tilde{v}_i \tilde{\boldsymbol{\psi}}_i$, the coefficients corresponding to the original basis $(\boldsymbol{\psi}_1, ..., \boldsymbol{\psi}_m)$ can then be calculated by:
    $v_i = \tilde{v}_i - \sum_{j=i+1}^m \langle\boldsymbol{\psi}_i, \tilde{\boldsymbol{\psi}}_j\rangle;\text{ and }
    v_m = \tilde{v}_m.$

Hence, after the aspect feature projection and decomposition, regardless of orthogonal or not, we have the following unique decomposition in space $\mathcal{V}$: $\boldsymbol{v} = \sum_{i=1}^m v_i \boldsymbol{\psi}_i$.

\begin{figure}[t]
	\centering
	\resizebox{0.6\linewidth}{!}{
		\includegraphics[width=0.5\textwidth]{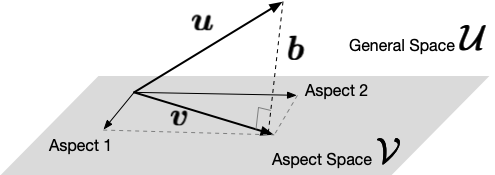}
	}
	\caption{An illustration of interpretable feature mapping. $\boldsymbol{u}$ is an uninterpretable feature in general space $\mathcal{U}$, and $\boldsymbol{v}$ is the interpretable projection of $\boldsymbol{u}$ in the interpretable aspect space $\mathcal{V}$. Here $\boldsymbol{b}$ indicates the difference between $\boldsymbol{u}$ and $\boldsymbol{v}$.}
	\label{fig:int_space}
\end{figure}

\begin{figure}[t]
	\centering
	\resizebox{\linewidth}{!}{
	\includegraphics[width=0.5\textwidth]{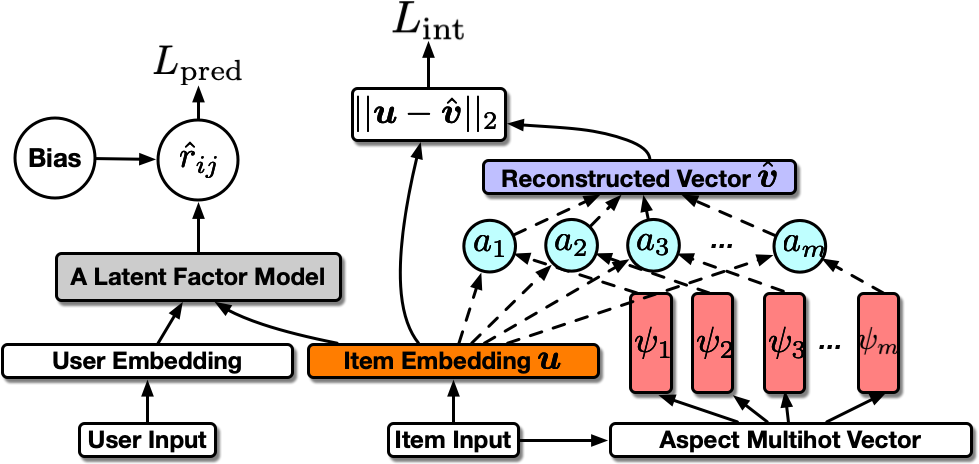}
	}
	\caption{The training phase: explainable recommendation via interpretable feature mapping.}
	\label{fig:amcf}
	
\end{figure}

\noindent \textbf{Aspect Projection via Attention}:
As described above, any interpretable aspect feature $\boldsymbol{v}$ can be uniquely decomposed as $\boldsymbol{v} = \sum_{i=1}^m v_i \boldsymbol{\psi}_i$, which is similar to the form of attention mechanism. Therefore, instead of using Gram-Schmidt orthogonalization process, we utilize attention mechanism to reconstruct $\boldsymbol{v}$ directly.  Assume we can obtain an attention vector $\boldsymbol{a} = (a_1, ..., a_m)$, which can be used to calculate $\hat{\boldsymbol{v}} = \sum_{i=1}^m a_i \boldsymbol{\psi}_i$, with the fact that the decomposition is unique, our goal is then to minimize the distance $||\hat{\boldsymbol{v}} - \boldsymbol{v}||_2$ to ensure that $a_i \approx v_i$.

However, as the interpretable aspect feature $\boldsymbol{v}$ is not available, we cannot minimize $||\hat{\boldsymbol{v}} - \boldsymbol{v}||_2$ directly. Fortunately, the general feature $\boldsymbol{u}$ is available (obtained from a base latent factor model), with the fact that $\boldsymbol{v}$ is the projection of $\boldsymbol{u}$, i.e. $\boldsymbol{v}=\boldsymbol{M}(\boldsymbol{u})$, we have the following lemma:
\begin{lem}\label{lm}
  Provided that $\boldsymbol{v}$ is the projection of $\boldsymbol{u}$ from space $\mathcal{U}$ to space $\mathcal{V}$, where $\mathcal{V}\subset \mathcal{U}$, we have
  \begin{align*}
      {\arg\min}_{\boldsymbol{a}} ||\hat{\boldsymbol{v}} - \boldsymbol{v}||_2 = {\arg\min}_{\boldsymbol{a}} ||\hat{\boldsymbol{v}} - \boldsymbol{u}||_2,
  \end{align*}
 where $\hat{\boldsymbol{v}} = \sum_{i=1}^m a_i \boldsymbol{\psi}_i$, $\boldsymbol{a} = (a_1,...,a_m)$, and $||\cdot||_2$ denotes $l_2$ norm.
\end{lem}

\begin{proof}
  Refer to the illustration in Figure~\ref{fig:int_space}, and denote the difference between $\boldsymbol{u}$ and $\boldsymbol{v}$ as $\boldsymbol{b}$, i.e. $\boldsymbol{u} = \boldsymbol{v} + \boldsymbol{b}$. Hence
  \begin{align*}
      {\arg\min}_{\boldsymbol{a}} ||\hat{\boldsymbol{v}} - \boldsymbol{u}||_2 =  {\arg\min}_{\boldsymbol{a}} ||\hat{\boldsymbol{v}} - \boldsymbol{v} - \boldsymbol{b}||_2.
  \end{align*}
  Note that $\boldsymbol{b}$ is perpendicular to $\hat{\boldsymbol{v}}$ and $\boldsymbol{v}$, the right hand side can then be written as 
  \begin{align*}
      {\arg\min}_{\boldsymbol{a}} ||\hat{\boldsymbol{v}} - \boldsymbol{v} - \boldsymbol{b}||_2 =  {\arg\min}_{\boldsymbol{a}} \sqrt{\left(||\hat{\boldsymbol{v}} - \boldsymbol{v}||_2^2 +||\boldsymbol{b}||_2^2\right)},
  \end{align*}
  as $\boldsymbol{b}$ is not parameterized by $\boldsymbol{a}$, we then get
  \begin{align*}
      {\arg\min}_{\boldsymbol{a}} ||\hat{\boldsymbol{v}} - \boldsymbol{v}||_2 = {\arg\min}_{\boldsymbol{a}} ||\hat{\boldsymbol{v}} - \boldsymbol{u}||_2.
  \end{align*}
\end{proof}
  
From the above proof, we know that attention mechanism is sufficient to reconstruct $\boldsymbol{v} \approx \boldsymbol{\hat{v}} = \sum_{i=1}^m a_i \boldsymbol{\psi}_i$ by minimizing $||\hat{\boldsymbol{v}} - \boldsymbol{u}||_2$. Note that from the perspective of disentanglement $\boldsymbol{u} = \boldsymbol{v} + \boldsymbol{b}$, the information in $\boldsymbol{b}$, i.e., the item specific characteristics, is not explained in our model. Intuitively, the item specific characteristics are learned from the metadata associated with the item.

\subsection{The Loss Function}
The loss function for finding the feature mapping $\boldsymbol{M}(\cdot)$ to achieve both interpretability and performance of the recommender model has $2$ components: 
\begin{itemize}
     \item $L_{pred}$ prediction loss in rating predictions, corresponding to the loss function for the base latent factor model.
     \item $L_{int}$ interpretation loss to the general feature $\boldsymbol{u}$. This loss is to quantify $||\hat{\boldsymbol{v}} - \boldsymbol{u}||_2$.
 \end{itemize}

We calculate the rating prediction loss component using RMSE:
$
    L_{pred} = \sqrt{\frac{1}{N}\sum_{(i,j)\in \text{Observed}} (r_{ij} - \hat{r}_{ij})^2},
$
where $\hat{r}_{ij}$ represents the predicted item ratings.
We then calculate the interpretation loss component as the average distance between $\boldsymbol{u}$ and $\hat{\boldsymbol{v}}$:
$
    L_{int}= \frac{1}{N}\sum_{(i,j)\in\text{Observed}}||\hat{\boldsymbol{v}} - \boldsymbol{u}||_2.
$
The loss component $L_{int}$ encourages the interpretable feature $\tilde{\boldsymbol{v}}$ obtained from the attentive neural network to be a good approximation of the aspect feature representation $\boldsymbol{v}$ (Lemma~\ref{lm}).
Hence the overall loss function is $L = L_{pred}  +\lambda L_{int}$,
where $\lambda$ is a tuning parameter to leverage importance between the two loss components. 

\noindent \textbf{Gradient Shielding Trick:} To ensure that interpretation doesn't compromise the prediction accuracy, we allow forward propagation to both $L_{int}$ and $L_{pred}$ but refrain the back-propagation from $L_{int}$ to the item embedding $\boldsymbol{u}$. In other words, when learning the model parameters based on back-propagation gradients, the item embedding $\boldsymbol{u}$ is updated only via the gradients from $L_{pred}$.

\subsection{User Preference Prediction}
Thus far we attempt to optimize the ability to predict user preference via aspect feature mapping. We call the user overall preference as \emph{general preference}, and the user preference on a specific item as \emph{specific preference}.

\noindent\textbf{General preference}: Figure~\ref{fig:general_pred} illustrates how to make prediction on user general preference. 
Here we define a virtual item $\boldsymbol{\hat{v}_k}$, which is a linear combination of aspect embeddings. For general preference, we let  $\hat{\boldsymbol{v}}_k = \boldsymbol{\psi}_k$ to simulate a pure aspect $k$ movie,
the resulting debiased (discarded all bias terms) rating prediction $\hat{p}_k$ indicates the user's preference on such specific aspect (e.g., positive for `like', negative for `dislike'). Formally:
	$\hat{p}_k = f(\boldsymbol{q}, \hat{\boldsymbol{v}}_k)$,
where $\boldsymbol{q}$ is the user embedding, $\hat{\boldsymbol{v}}_k = \boldsymbol{\psi}_k$  is the aspect $k$'s embedding, and $f(\cdot)$ is the corresponding base latent factor model without bias terms.

\begin{figure}[t]
	\centering
	\resizebox{0.55\linewidth}{!}{
	\includegraphics[width=0.5\textwidth]{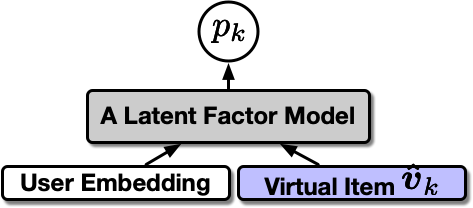}
	}
	\caption{The explanation phase: a virtual item vector $\boldsymbol{\hat{v}}_k$ is calculated to represent a specific aspect.}
	\label{fig:general_pred}
\end{figure}

\noindent\textbf{Specific preference}: Figure~\ref{fig:general_pred} also shows our model's ability to predict user preference on a specific item,
as long as we can find how to represent them in terms of aspect embeddings. Fortunately, the attention mechanism is able to help us find the constitution of any item in terms of aspect embeddings using the attention weights. That is, for any item, it is possible to rewrite the latent representation $\boldsymbol{u}$ as a linear combination of aspect embeddings:
$
    \boldsymbol{u} = \boldsymbol{\hat{v}} + \boldsymbol{b} =  \sum_k a_k \boldsymbol{\psi}_k + \boldsymbol{b},
$
where $a_k$ and $\boldsymbol{\psi}_k$ are the $k$-th attention weight and the $k$-th aspect feature, respectively. The term $\boldsymbol{b}$ reflects interpretation loss.
For aspect $k$ of an item, we use $\boldsymbol{\hat{v}}_k = a_k\boldsymbol{\psi}_k$ to represent the embedding of a virtual item which represents the aspect $k$ property of the specific item. Hence, the output $\hat{p}_k$ indicates the specific preference on aspect $k$ of a specific item.

\noindent\textbf{Model Interpretability}: From specific preference, a learned latent general feature can be decomposed into the linear combination of interpretable aspect features, which would help interpret models in a more explicit and systematic manner.


\section{Experiments and Discussion}\label{sec:exp}
We design and perform experiments to demonstrate two advantages of our AMCF approach: 1) comparable rating predictions; 2) good explanations on why a user likes/dislikes an item. To demonstrate the first advantage we compare the rating prediction performance with  baseline approaches of \emph{rating prediction only} methods.
The demonstration of the second advantage, however, is not a trivial task since currently no gold standard for evaluating explanation of recommendations except for using real customer feedback\cite{chen2019co-attentive,gao2019explainable}. Hence it's necessary to develop new schemes to evaluate the quality of explainability for both general and specific user preferences.


\subsection{Datasets}
\noindent\textbf{MovieLens Datasets} This data set \cite{harper2016movielens} offers very complete movie genre information, which provides a perfect foundation for genre (aspect) preference prediction, i.e. determining which genre a user likes most. We consider the $18$ movie genres as aspects. 
\\
\noindent\textbf{Yahoo Movies Dataset} This data set from Yahoo Lab contains usual user-movie ratings as well as metadata such as movie's title, release date, genre, directors and actors. 
We use the $29$ movie genres as the aspects for movie recommendation and explanation. Summary statistics are shown in Table~\ref{tb:dataset}.
\\
\noindent{\textbf{Pre-processing}:} We use multi-hot encoding to represent genres of each movie or book, where $1$ indicates the movie is of that genre, $0$ otherwise. However, there are still plenty of movies with missing genre information, in such cases, we simply set them as none of any listed genre, i.e., all zeros in the aspect multi-hot vector: $(0,0,...,0)$. 

\begin{table}
	\centering
	\resizebox{\linewidth}{!}{
	\begin{tabular}{|l|r|r|r|r|}
		\hline
		Dataset   & \# of ratings & \# of items &\# of users &  \# of genres \\ \hline
		MovieLens 1M & 1,000,209     & 3,706  & 6,040   &18 \\ \hline
		MovieLens 100k & 100,000     & 1,682  & 943   &18\\ \hline
		Yahoo Movie & 211,333     & 11,915  & 7,642   &29\\ \hline

	\end{tabular}
	}
	\caption{Summary statistics of the data sets.}
	\label{tb:dataset}
\end{table}

\begin{table*}[th]
\Large
\centering
\resizebox{\linewidth}{!}{
\begin{tabular}{|c|r|r|r|r|r|r|r|r|r|r|r|r|r|r|r|r|r|r|r|}
\hline
\multirow{2}{*}{Dataset} & \multicolumn{1}{c|}{\multirow{2}{*}{LR}} & \multicolumn{3}{c|}{SVD} & \multicolumn{3}{c|}{AMCF(SVD)} & \multicolumn{3}{c|}{NCF} & \multicolumn{3}{c|}{AMCF(NCF)} & \multicolumn{3}{c|}{FM} & \multicolumn{3}{c|}{AMCF(FM)} \\ \cline{3-20} 
 & \multicolumn{1}{c|}{} & 20 & 80 & 120 & 20 & 80 & 120 & 20 & 80 & 120 & 20 & 80 & 120 & 20 & 80 & 120 & 20 & 80 & 120 \\ \hline
ML100K & 1.018 & 0.908 & 0.908 & \textbf{0.907} & 0.907 & 0.909 & \textbf{0.907} & 0.939 & 0.939 & 0.936 & 0.937 & 0.939 & 0.934 & 0.937 & 0.933 & 0.929 & 0.940 & 0.936 & 0.931 \\ \hline
ML1M & 1.032 & 0.861 & 0.860 & 0.853 & 0.860 & 0.858 & \textbf{0.851} & 0.900 & 0.895 & 0.892 & 0.902 & 0.889 & 0.889 & 0.915 & 0.914 & 0.913 & 0.915 & 0.914 & 0.915 \\ \hline
Yahoo & 1.119 & 1.022 & 1.021 & 1.014 & 1.022 & 1.022 & \textbf{1.010} & 1.028 & 1.027 & 1.028 & 1.027 & 1.026 & 1.025 & 1.042 & 1.042 & 1.039 & 1.044 & 1.042 & 1.041 \\ \hline
\end{tabular}
}

	\caption{Performance comparison of rating prediction using different data sets in terms of RMSE. Texts in the parentheses indicate the base CF models that we choose for AMCF; and the numbers [20, 80, 120] indicate the dimension of the latent factors for the models. }
	\label{tab:result}
\end{table*}

\begin{table}[ht]
	\centering
	\resizebox{\linewidth}{!}{
\begin{tabular}{|c|c|c|c|c|c|c|}
\hline
Dataset & Model & T1@3 & B1@3 & T3@5 & B3@5 & $score_s$ \\ \hline
\multirow{3}{*}{ML100K} & AMCF & 0.500 & 0.481 & 0.538 & 0.553 & \textbf{0.378} \\ \cline{2-7} 
 & LR & \textbf{0.628} & \textbf{0.668} & \textbf{0.637} & \textbf{0.675} & 0.371 \\ \cline{2-7} 
 & Rand & 0.167 & 0.167 & 0.278 & 0.278 & 0 \\ \hline
\multirow{3}{*}{ML1M} & AMCF & 0.461 & 0.403 & 0.513 & 0.489 & \textbf{0.353} \\ \cline{2-7} 
 & LR & \textbf{0.572} & \textbf{0.565} & \textbf{0.598} & \textbf{0.620} & 0.322 \\ \cline{2-7} 
 & Rand & 0.167 & 0.167 & 0.278 & 0.278 & 0 \\ \hline
\multirow{3}{*}{Yahoo} & AMCF & 0.413 & 0.409 & 0.422 & 0.440 & 0.224 \\ \cline{2-7} 
 & LR & \textbf{0.630} & \textbf{0.648} & \textbf{0.628} & \textbf{0.565} & \textbf{0.235} \\ \cline{2-7} 
 & Rand & 0.103 & 0.103 & 0.172 & 0.172 & 0 \\ \hline
\end{tabular}
	}
	\caption{Preferences outputs: TM@K/BM@K represent Top/ Bottom M recall at K, and $score_s$ represents the specific preference. The Rand rows show the theoretical random preference outputs. Here AMCF takes SVD with 120 latent factors as the base model.}
	\label{tb:gen}
\end{table}

\subsection{Results of Prediction Accuracy}\label{sec:baseline}
We select several strong baseline models to compare rating prediction accuracy, including non-interpretable models, such as SVD \cite{koren2009matrix}, Neural Collaborative Filtering (NCF) \cite{he2017neural} and Factorization Machine (FM) \cite{rendle2010factorization}, and an interpretable linear regression model (LR). Here the LR model is implemented by using aspects as inputs and learning separate parameter sets for different individual users. In comparison, our AMCF approaches also include SVD, NCF or FM as the base model to demonstrate that the interpretation module doesn't compromise the prediction accuracy. Note that since regular NCF and FM are designed for implicit ratings (1 and 0), we replace their last sigmoid output layers with fully connected layers in order to output explicit ratings. 

In terms of robustness, we set the dimension of latent factors in the base models to $20, 80$, and $120$. The regularization tuning parameter $\lambda$ is set to $0.05$, which demonstrated better performance compared to other selections. 
It is worth noting that the tuning parameters of the base model of our AMCF approach are directly inherited from the corresponding non-interpretable model. We compare our AMCF models with baseline models as shown in Table~\ref{tab:result}. It is clear that AMCF achieves comparable prediction accuracy to their non-interpretable counterparts, and significantly outperforms the interpretable LR model.




\subsection{Evaluation of Explainability}\label{sec:exev}
Despite the recent efforts have been made to evaluate the quality of explanation by defining explainability precision (EP) and explainability recall (ER)\cite{peake2018explanation,abdollahi2016explainable}, the scarcity of ground truth such as a user's true preference remains a significant obstacle for explainable recommendation. \cite{gao2019explainable} make an initial effort in collecting ground truth by surveying real customers, however, the labor intense, time consuming and sampling bias may prevent its large-scale applications in a variety of contexts.
Other text-based approaches \cite{costa2018automatic,lu2018like} can also use natural language processing (NLP) metrics such as Automated Readability Index (ARI) and Flesch Reading Ease (FRE). 
As we don't use metadata such as text reviews in our AMCF model, user review based explanation and evaluation could be a potential future extension to our model. 

Here we develop novel quantitative evaluation schemes to assess our model's explanation quality in terms of general preferences and specific preferences, respectively. 


\subsubsection{General Preference}
Let's denote the ground truth of user general preferences as $\boldsymbol{p}_i$ for user $i$, and the model's predicted preference for user $i$ is $\hat{\boldsymbol{p}}_i$, we propose measures inspired by  \emph{Recall@K} in recommendation evaluations.

\textbf{Top $M$ recall at $K$} (TM@K): Given the $M$ most preferred aspects of a user $i$ from $\boldsymbol{p}_i$, top $M$ recall at $K$ is defined as the ratio of the $M$ aspects located in the top $K$ highest valued aspects in $\hat{\boldsymbol{p}}_i$. For example, if $\boldsymbol{p}_i$ indicates that user $i$'s top $3$ preferred aspects are \emph{Adventure, Drama}, and \emph{Thriller}, while the predicted $\hat{\boldsymbol{p}}_i$ shows that the top $5$ are \emph{Adventure, Comedy, Children, Drama, Crime}, the top $3$ recalls at $5$ (T3@5) is then $2/3$ whereas top $1$ recall at $3$ (T1@3) is $1$.

\textbf{Bottom $M$ recall at $K$} (BM@K): Similarly defined as above, except that it measures the most disliked aspects. 

As the ground truth of user preferences are usually not available, some reasonable approximations are needed. 
Hence we propose a method to calculate the so-called surrogate ground truth. First we define the weights $\boldsymbol{w}_{ij} = (r_{ij} - b^u_i - b^v_j - \bar{r}) / A $, where the weight $\boldsymbol{w}_{ij}$ is calculated by nullifying user bias $b^u_i$, item bias $b^v_j$, and global average $\bar{r}$, and $A$ is a constant indicating the maximum rating (e.g. $A=5$ for most datasets). Note that user bias $b^u_i$ and item bias $b^v_j$ can be easily calculated by $b^u_i = (\frac{1}{|V_i|}\sum_{j\in V_i}r_{ij}) - \bar{r}$, and $b^v_j = (\frac{1}{|U_j|}\sum_{i\in U_j}r_{ij}) - \bar{r}$. Here $V_i$ represents the sets of items rated by user $i$, and $U_j$ represents the sets of users that have rated item $j$. With the weights we calculate user $i$'s preference on aspect $t$ using the following formula:
    $\boldsymbol{p}_{i}^t = \sum_{j\in V_i}\boldsymbol{w}_{ij} s_{j}^t,$
where $s_j^t = 1$ if item $j$ has aspect $t$, $0$ otherwise. Hence a user $i$'s overall preference can be represented by an $l_1$ normalized vector $\boldsymbol{p}_i = (\boldsymbol{p}_{i}^1,...,\boldsymbol{p}_{i}^t,...,\boldsymbol{p}_{i}^T)/||\boldsymbol{p}_i||_1$.
As our model can output a user preference vector directly, we evaluate the explainability by calculating the average of TM@K and BM@K.
The evaluation results are reported in Table~\ref{tb:gen}. We observe that the explainability of AMCF is significantly better than random interpretation, and is comparable to the strong interpretable baseline LR model with much better prediction accuracy. Thus our AMCF model successfully integrates the strong prediction performance of a latent factor model and the strong interpretability of a LR model.

\subsubsection{Specific Preference}
Our approach is also capable of predicting a user's preference on a specific item, i.e. $\hat{\boldsymbol{p}}_{ij}$, showing which aspects of item $j$ are liked/disliked by the user $i$. Compared to user general preference across all items, the problem of which aspect of an item attracts the user most (specific preference) is more interesting and more challenging.
There is no widely accepted strategy to evaluate the quality of single item preference prediction (except for direct customer survey). 
Here we propose a simple yet effective evaluation scheme to illustrate the quality of our model's explanation on user specific preference. With the overall preference $\boldsymbol{p}_i$ of user $i$ given above, and assuming $\boldsymbol{s}_j$ is the multi-hot vector represents the aspects of item $j$, we say the element-wise product $\boldsymbol{p}_{ij} = \boldsymbol{p}_i \odot \boldsymbol{s}_j$ reflects the user's specific preference on item $j$. 

Note that we should not use the TM@K/BM@K scheme as in general preference evaluation, both $\boldsymbol{p}_{ij}$ and predicted $\boldsymbol{\hat{p}}_{ij}$'s entries are mostly zeros, since each movie is only categorized into a few genres. Hence the quality of specific preference prediction is expressed using a similarity measure. We use $s(\boldsymbol{p}_{ij}, \hat{\boldsymbol{p}}_{ij})$ to represent the cosine similarity between $\boldsymbol{p}_{ij}$ and $\hat{\boldsymbol{p}}_{ij}$, and the score for specific preference prediction is defined by averaging over all user-item pairs in the test set:
    $score_s = \frac{1}{N}\sum_{ij} s(\boldsymbol{p}_{ij}, \hat{\boldsymbol{p}}_{ij}).$
We report the results of specific user preferences in the $score_s$ column of Table~\ref{tb:gen}. As the LR cannot give specific user preferences directly, we simply apply
 $\boldsymbol{\hat{p}}_{ij} = \boldsymbol{\hat{p}}_i \odot \boldsymbol{s}_j$ where $\boldsymbol{\hat{p}}_i$ represents the general preference predicted by LR.

\begin{figure}[t]
	\centering
	\resizebox{1\linewidth}{!}{
		\includegraphics[width=\textwidth]{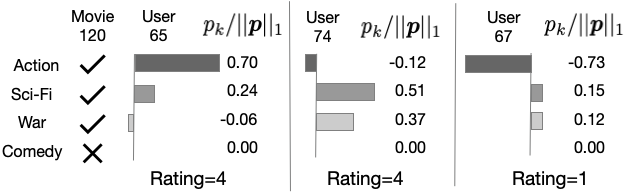}
	}
	\caption{Examples of explainable recommendations. We $l_1$-normalize the preference vector $\boldsymbol{p}$ to make the comparison fair. }
	\label{fig:ex}
\end{figure}

\textbf{An insight}: Assume that for a specific user $i$, our AMCF model can be simply written as  $\hat{r}_{ij} = f_i(\boldsymbol{u}_j)$ to predict the rating for item $j$. Note that our AMCF model can decompose the item in terms of aspects. Lets denote these aspects as $\{\psi_1, ... \psi_m\}$. Then the prediction can be approximated by $\hat{r}_{ij} \approx f_i(\sum_{k=1}^m a_{jk} \psi_k)$, where $a_{jk}$ denote the $k$-th attention weights for item $j$. In the case of LR, the rating is obtained by $\hat{r}_{ij} = g_i(\sum_{k=1}^m b_{k} x_k)$, where $g_i$ is the LR model for user $i$,  $b_{k}$ is the $k$-th coefficient of it, and $x_k$ represents the indicator of aspect $k$, $x_k=1$ when the item has aspect $k$, $x_k=0$ otherwise. The similarity between AMCF formula and LR formula listed above indicates that the coefficients of LR and the preference output of AMCF share the same intrinsic meaning, i.e., both indicate the importance of aspects.

\textbf{An example}:
For specific explanation, given a user $i$ and an item $j$, our AMCF model predicts a vector $\boldsymbol{p}$, representing the user $i$'s specific preference on an item $j$ in terms of all predefined aspects. Specifically, the magnitude of each entry of $\boldsymbol{p}$ (i.e. $|p_i|$) represents the impact of a specific aspect on whether an item liked by a user or not. For example, in Figure~\ref{fig:ex}, the movie $120$ is high-rated by both users $65$ and $74$, however, with differential explanations: the former user preference is more on the \textit{Action} genre whereas the latter is more on \textit{Sci-Fi} and \textit{War}. On the other hand, the same movie is low-rated by user $67$ mainly due to the dislike of \textit{Action} genre. 

\section{Conclusion}
Modelers tend to better appreciate the interpretable recommender systems whereas users are more likely to accept the explainable recommendations. In this paper, we proposed a novel interpretable feature mapping strategy attempting to achieve both goals: systems interpretability and recommendation explainability. Using extensive experiments and tailor-made evaluation schemes, our AMCF method demonstrates strong performance in both recommendation and explanation.

\section{Acknowledgement}
This work is supported by the National Science Foundation under grant no. IIS-1724227.

\bibliographystyle{named}
\bibliography{ref}

\begin{thebibliography}{}

\bibitem[\protect\citeauthoryear{Abdollahi and
  Nasraoui}{2016}]{abdollahi2016explainable}
Behnoush Abdollahi and Olfa Nasraoui.
\newblock Explainable matrix factorization for collaborative filtering.
\newblock In {\em Proceedings of the 25th WWW}, pages 5--6. International WWW
  Conferences Steering Committee, 2016.

\bibitem[\protect\citeauthoryear{Baral \bgroup \em et al.\egroup
  }{2018}]{baral2018reel}
Ramesh Baral, XiaoLong Zhu, SS~Iyengar, and Tao Li.
\newblock Reel: R eview aware explanation of location recommendation.
\newblock In {\em Proceedings of the 26th Conference on User Modeling,
  Adaptation and Personalization}, pages 23--32. ACM, 2018.

\bibitem[\protect\citeauthoryear{Bauman \bgroup \em et al.\egroup
  }{2017}]{bauman2017aspect}
Konstantin Bauman, Bing Liu, and Alexander Tuzhilin.
\newblock Aspect based recommendations: Recommending items with the most
  valuable aspects based on user reviews.
\newblock In {\em Proceedings of the 23rd ACM SIGKDD}, pages 717--725. ACM,
  2017.

\bibitem[\protect\citeauthoryear{Chen \bgroup \em et al.\egroup
  }{2016}]{chen2016learning}
Xu~Chen, Zheng Qin, Yongfeng Zhang, and Tao Xu.
\newblock Learning to rank features for recommendation over multiple
  categories.
\newblock In {\em Proceedings of the 39th International ACM SIGIR}, pages
  305--314. ACM, 2016.

\bibitem[\protect\citeauthoryear{Chen \bgroup \em et al.\egroup
  }{2018}]{chen2018neural}
Chong Chen, Min Zhang, Yiqun Liu, and Shaoping Ma.
\newblock Neural attentional rating regression with review-level explanations.
\newblock In {\em Proceedings of the 2018 WWW}, pages 1583--1592. International
  WWW Conferences Steering Committee, 2018.

\bibitem[\protect\citeauthoryear{Chen \bgroup \em et al.\egroup
  }{2019a}]{chen2019dynamic}
Xu~Chen, Yongfeng Zhang, and Zheng Qin.
\newblock Dynamic explainable recommendation based on neural attentive models.
\newblock In {\em Proceedings of the AAAI Conference on Artificial
  Intelligence}, volume~33, pages 53--60, 2019.

\bibitem[\protect\citeauthoryear{Chen \bgroup \em et al.\egroup
  }{2019b}]{chen2019co-attentive}
Zhongxia Chen, Xiting Wang, Xing Xie, Tong Wu, Guoqin Bu, Yining Wang, and
  Enhong Chen.
\newblock Co-attentive multi-task learning for explainable recommendation.
\newblock In {\em IJCAI}, June 2019.

\bibitem[\protect\citeauthoryear{Costa \bgroup \em et al.\egroup
  }{2018}]{costa2018automatic}
Felipe Costa, Sixun Ouyang, Peter Dolog, and Aonghus Lawlor.
\newblock Automatic generation of natural language explanations.
\newblock In {\em Proceedings of the 23rd International Conference on
  Intelligent User Interfaces Companion}, page~57. ACM, 2018.

\bibitem[\protect\citeauthoryear{Diao \bgroup \em et al.\egroup
  }{2014}]{diao2014jointly}
Qiming Diao, Minghui Qiu, Chao-Yuan Wu, Alexander~J Smola, Jing Jiang, and
  Chong Wang.
\newblock Jointly modeling aspects, ratings and sentiments for movie
  recommendation (jmars).
\newblock In {\em Proceedings of the 20th ACM SIGKDD}, pages 193--202. ACM,
  2014.

\bibitem[\protect\citeauthoryear{Gao \bgroup \em et al.\egroup
  }{2019}]{gao2019explainable}
Jingyue Gao, Xiting Wang, Yasha Wang, and Xing Xie.
\newblock Explainable recommendation through attentive multi-view learning.
\newblock In {\em AAAI Conference on Artificial Intelligence (AAAI)}, March
  2019.

\bibitem[\protect\citeauthoryear{Harper and
  Konstan}{2016}]{harper2016movielens}
F~Maxwell Harper and Joseph~A Konstan.
\newblock The movielens datasets: History and context.
\newblock {\em Acm transactions on interactive intelligent systems (tiis)},
  5(4):19, 2016.

\bibitem[\protect\citeauthoryear{He \bgroup \em et al.\egroup
  }{2015}]{he2015trirank}
Xiangnan He, Tao Chen, Min-Yen Kan, and Xiao Chen.
\newblock Trirank: Review-aware explainable recommendation by modeling aspects.
\newblock In {\em Proceedings of the 24th ACM International on Conference on
  Information and Knowledge Management}, pages 1661--1670. ACM, 2015.

\bibitem[\protect\citeauthoryear{He \bgroup \em et al.\egroup
  }{2017}]{he2017neural}
Xiangnan He, Lizi Liao, Hanwang Zhang, Liqiang Nie, Xia Hu, and Tat-Seng Chua.
\newblock Neural collaborative filtering.
\newblock In {\em Proceedings of the 26th WWW}, pages 173--182. International
  WWW Conferences Steering Committee, 2017.

\bibitem[\protect\citeauthoryear{Hou \bgroup \em et al.\egroup
  }{2019}]{hou2019explainable}
Yunfeng Hou, Ning Yang, Yi~Wu, and S~Yu Philip.
\newblock Explainable recommendation with fusion of aspect information.
\newblock {\em WWW}, 22(1):221--240, 2019.

\bibitem[\protect\citeauthoryear{Koren \bgroup \em et al.\egroup
  }{2009}]{koren2009matrix}
Yehuda Koren, Robert Bell, and Chris Volinsky.
\newblock Matrix factorization techniques for recommender systems.
\newblock {\em Computer}, 8:30--37, 2009.

\bibitem[\protect\citeauthoryear{Lee and Jung}{2018}]{lee2018explainable}
O-Joun Lee and Jason~J Jung.
\newblock Explainable movie recommendation systems by using story-based
  similarity.
\newblock In {\em IUI Workshops}, 2018.

\bibitem[\protect\citeauthoryear{Lu \bgroup \em et al.\egroup
  }{2018}]{lu2018like}
Yichao Lu, Ruihai Dong, and Barry Smyth.
\newblock Why i like it: multi-task learning for recommendation and
  explanation.
\newblock In {\em Proceedings of the 12th ACM Conference on Recommender
  Systems}, pages 4--12. ACM, 2018.

\bibitem[\protect\citeauthoryear{Peake and Wang}{2018}]{peake2018explanation}
Georgina Peake and Jun Wang.
\newblock Explanation mining: Post hoc interpretability of latent factor models
  for recommendation systems.
\newblock In {\em Proceedings of the 24th ACM SIGKDD}, pages 2060--2069. ACM,
  2018.

\bibitem[\protect\citeauthoryear{Rendle}{2010}]{rendle2010factorization}
Steffen Rendle.
\newblock Factorization machines.
\newblock In {\em 2010 IEEE International Conference on Data Mining}, pages
  995--1000. IEEE, 2010.

\bibitem[\protect\citeauthoryear{Ribeiro \bgroup \em et al.\egroup
  }{2016}]{ribeiro2016should}
Marco~Tulio Ribeiro, Sameer Singh, and Carlos Guestrin.
\newblock Why should i trust you?: Explaining the predictions of any
  classifier.
\newblock In {\em Proceedings of the 22nd ACM SIGKDD}, pages 1135--1144. ACM,
  2016.

\bibitem[\protect\citeauthoryear{Wang \bgroup \em et al.\egroup
  }{2019}]{wang2019explainable}
Xiang Wang, Dingxian Wang, Canran Xu, Xiangnan He, Yixin Cao, and Tat-Seng
  Chua.
\newblock Explainable reasoning over knowledge graphs for recommendation.
\newblock In {\em Proceedings of the AAAI Conference on Artificial
  Intelligence}, volume~33, pages 5329--5336, 2019.

\bibitem[\protect\citeauthoryear{Wu \bgroup \em et al.\egroup
  }{2019}]{wu2019context}
Libing Wu, Cong Quan, Chenliang Li, Qian Wang, Bolong Zheng, and Xiangyang Luo.
\newblock A context-aware user-item representation learning for item
  recommendation.
\newblock {\em ACM Transactions on Information Systems (TOIS)}, 37(2):22, 2019.

\bibitem[\protect\citeauthoryear{Zhang and Chen}{2018}]{zhang2018explainable}
Yongfeng Zhang and Xu~Chen.
\newblock Explainable recommendation: A survey and new perspectives.
\newblock {\em arXiv preprint arXiv:1804.11192}, 2018.

\bibitem[\protect\citeauthoryear{Zhang \bgroup \em et al.\egroup
  }{2014}]{zhang2014explicit}
Yongfeng Zhang, Guokun Lai, Min Zhang, Yi~Zhang, Yiqun Liu, and Shaoping Ma.
\newblock Explicit factor models for explainable recommendation based on
  phrase-level sentiment analysis.
\newblock In {\em Proceedings of the 37th international ACM SIGIR}, pages
  83--92. ACM, 2014.

\bibitem[\protect\citeauthoryear{Zhang \bgroup \em et al.\egroup
  }{2019}]{zhang201919}
Yongfeng Zhang, Jiaxin Mao, and Qingyao Ai.
\newblock Www’19 tutorial on explainable recommendation and search.
\newblock In {\em Companion Proceedings of WWW}, pages 1330--1331. ACM, 2019.

\end{thebibliography}

\end{document}